\documentclass[letterpaper, 10 pt, conference]{ieeeconf}
\IEEEoverridecommandlockouts 
\usepackage{dsfont}

\usepackage{amsthm}
\newtheorem{definition}{Definition}
\newtheorem{theorem}{Theorem}
\newtheorem{lemma}{Lemma}

\usepackage{amsmath}
\usepackage{amssymb}

\usepackage{algorithm}
\usepackage{algorithmic}

\usepackage{placeins}
\usepackage{graphicx}
\usepackage{float}
\usepackage{subfigure} 
\usepackage{hyperref}
\usepackage{wrapfig}
\title{\LARGE \bf
Reinforcement Learning for Safe Robot Control \\ using Control Lyapunov Barrier Functions
}

\author{Desong Du$^{*1,2}$, Shaohang Han$^{*2}$, Naiming Qi$^{1}$, Haitham Bou Ammar$^{3,4}$, Jun Wang$^{4}$ and Wei Pan$^{5,2}$
\thanks{*indicates equal contribution}
\thanks{The work is supported by Huawei and China Scholarship Council No.202006120130.}
\thanks{$^{1}$School of Astronautics, Harbin Institute of Technology, China. $^{2}$Department of Cognitive Robotics, Delft University of Technology, Netherlands. $^{3}$Huawei Technologies, United Kingdom. $^{4}$Department of Computer Science, University College London, United Kingdom. $^{5}$Department of Computer Science, University of Manchester, United Kingdom.}%
}

\begin{document}

\maketitle
\thispagestyle{empty}
\pagestyle{empty}

\begin{abstract}

Reinforcement learning (RL) exhibits impressive performance when managing complicated control tasks for robots. However, its wide application to physical robots is limited by the absence of strong safety guarantees. To overcome this challenge, this paper explores the control Lyapunov barrier function (CLBF) to analyze the safety and reachability solely based on data without explicitly employing a dynamic model. We also proposed the Lyapunov barrier actor-critic (LBAC), a model-free RL algorithm, to search for a controller that satisfies the data-based approximation of the safety and reachability conditions. The proposed approach is demonstrated through simulation and real-world robot control experiments, i.e., a 2D quadrotor navigation task. The experimental findings reveal this approach's effectiveness in reachability and safety, surpassing other model-free RL methods.

\end{abstract}

\section{Introduction}
Reinforcement learning (RL) has achieved impressive and promising results in robotics, such as manipulation \cite{nagabandi2020deep}, unmanned vehicle navigation \cite{kahn2017uncertainty}, drone flight \cite{lambert2019low,Belkhale2021model}, etc., thanks to its ability of handling intricate models and adapting to diverse problem scenarios with ease. Meanwhile, a safe control policy is imperative for a robot in the real world, as dangerous behaviors can cause irreparable damage or costly losses. Therefore, the RL methods that can provide a safety guarantee for robot control have received considerable interest and progress \cite{achiam2017constrained,thomas2021safe,tessler2018reward,cheng2019end,pham2018optlayer,ohnishi2019barrier}.

A recent line of work focuses on designing novel RL algorithms, e.g., actor-critic, for constrained Markov Decision Process (CMDP). In these methods, the system encourages the satisfaction of the constraints by adding a constant penalty to the objective function \cite{thomas2021safe} or constructing safety critics while doing policy optimization in a multi-objective manner \cite{achiam2017constrained,tessler2018reward, srinivasan2020learning, shen2014risk}. Although these approaches are attractive for their generality and simplicity, they either need model \cite{thomas2021safe}, or only encourage the safety constraints to be satisfied probabilistically.

\begin{figure}[htp] 
    \centering
    \vspace{-0.2cm}
    \includegraphics[width=0.45\textwidth]{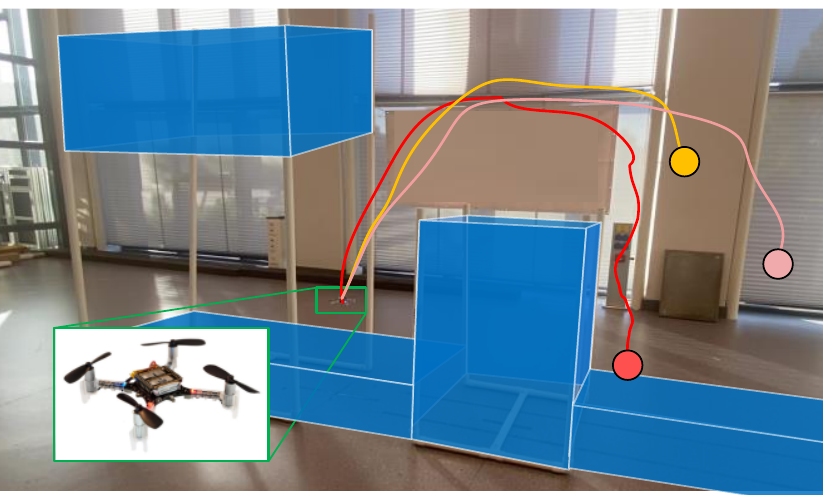}
    \caption{The 2D quadrotor navigation task. Lines stand for trajectories. The circles are the initial position. The blue regions represent obstacles. Video is available at \url{https://youtu.be/_8Yr_QRRYik}.
    }
    \label{fig:experiment_setup}
\end{figure}

An alternative type of methods focuses on reachability and safety guarantee (sufficient conditions) by constructing/learning control Lyapunov functions (CLF) and control barrier functions (CBF) that can respectively certify the reachability and safety \cite{cheng2019end,ohnishi2019barrier,taylor2019episodic,Grandia2021multi,Nguyen20163D,gurriet2018towards,jagtap2020control,choi2020reinforcement}. The relevant safe controllers are normally designed by adding a safety filter to a reference controller, such as a RL controller \cite{cheng2019end,ohnishi2019barrier,taylor2019episodic}, a model predictive control (MPC) controller \cite{Grandia2021multi}, etc. Unfortunately, these approaches have two disadvantages: (1) there might be conflicts between CLFs and CBFs as separate certificates \cite{dawson2022safe, meng2022smooth} (see Figure~\ref{fig:clf-cbf} in Section~\ref{sec:clfcbf}); (2) the CLFs and CBFs are generally non-trivial to find \cite{dawson2022safe}, especially for nonlinear systems. Even though there are learning methods to find CLFs and CBFs, knowledge of dynamic models has to be explicitly used \cite{jin2020neural}.

In this paper, we propose a data-based reachability and safety theorem without explicitly using the knowledge of a dynamic system model. The contribution of this paper can be summarized as follows:  
 (1) we used samples to approximate the critic as a control Lyapunov barrier function (CLBF), a single unified certificate, which is parameterized by deep neural networks, so as to guarantee both reachability and safety. The corresponding actor is a controller that satisfies both the reachability and safety guarantees.
 (2) we deploy the learned controller to a real-world robot, i.e., a Crazyflie 2.0 quadrotor, for a 2D quadrotor navigation task. The 2D quadrotor navigation task is shown as in Figure~\ref{fig:experiment_setup}. The experiments show our approach has better performance than other model-free RL methods. Our approach, by using CLBFs, can avoid conflicts between the CLFs and CBFs certificates. Compared to the model-based approaches that learn CLBFs using supervised learning \cite{dawson2022safe} or handcraft CLBFs \cite{romdlony2016}, our method does not need the knowledge of models explicitly.

\section{Related Works}

Prior work has studied safety in RL in several ways, including imposing constraints on expected return \cite{achiam2017constrained,tessler2018reward}, risk measures such as Conditional Value at Risk and percentile estimates \cite{shen2014risk,geibel2005risk,tamar2014policy}, and avoiding regions where constraints are violated \cite{berkenkamp2017safe,turchetta2016safe,thananjeyan2021recovery}. This paper focuses on the reach-avoid problem that belongs to the last situation.

To solve the reach-avoid problem, a popular strategy involves modifying the policy optimization procedure of standard RL algorithms to reason about task rewards and constraints simultaneously. One method is constrained policy optimization (CPO), which adds a constraint-related cost to the policy objective \cite{achiam2017constrained}. Another type of method tries to optimize a Lagrangian relaxation \cite{tessler2018reward,srinivasan2020learning,geibel2005risk,thananjeyan2021recovery,bharadhwaj2020conservative}. They normally use a safety critic to ensure safety, but this separate critic can only evaluate risk in a probabilistic way.
Other methods involve constructing Lyapunov functions for the unsafe region \cite{chow2018lyapunov,chow2019lyapunov}. However, these approaches require a baseline policy that already satisfies the constraints.

\section{Preliminaries and Background}

In RL for safe control, the dynamical system is typically characterized by CMDP $\hat{M}=(\mathcal{S}, \mathcal{A}, P, c, \gamma, \mathcal{I})$ \cite{altman1999constrained}. $s_{t} \in  \mathcal{S} \subseteq \mathbb{R}^{n}$ is the state vector at time $t$, $S$ denotes the state space. The agent then takes an action $a_{t} \in  \mathcal{A} \subseteq  \mathbb{R}^{m}$ according to a stochastic policy/controller $\pi\left(a_{t} \mid s_{t}\right)$. The transition of the state is dominated by the transition probability density function $P\left(s_{t+1} \mid s_{t}, a_{t}\right)$, which denotes the probability density of the next state $s_{t+1}$. A cost function $c(s_{t},a_{t})$ is used to measure the immediate performance of a state-action pair $(s_{t},a_{t})$, and $\mathcal{I}(s_{t})$ indicates whether the state violates the safety constraints or not. The goal is to find $\pi^{*}$ that can minimize the objective function return the expected return $J(\pi) \triangleq \sum_{t=1}^{\infty} \mathbb{E}_{s_{t}, a_{t}} \gamma^{t} c\left(s_{t}, a_{t}\right)$ with the discount factor $\gamma \in[0,1)$, and $\forall t \in \mathbb{Z}_{+}, \mathcal{I}(s_t)=0$. 
Moreover, some notations are to be defined. The closed-loop state distribution at a certain instant $t$ as $p(s \mid \rho, \pi, t)$, which can be defined iteratively: $p\left(s^{\prime} \mid \rho, \pi, t+1\right)=\int_{\mathcal{S}} P(s^{\prime}|s,\pi(s)) p(s \mid \rho, \pi, t) \mathrm{d} s, \forall t \in \mathbb{Z}_{+}$ and $p(s \mid \rho, \pi, 0) = \rho(s)$. 

In this paper, we focus on the reach-avoid problems, in which the agent reaches the goal condition and avoids certain unsafe conditions. It is defined as follows:
\begin{definition}
\label{def:state}
(Reach-Avoid Problem). In a CMDP setting with a goal configuration $s_{\text{goal}}$ and a set of unsafe states $\mathcal{S}_{\text{unsafe}} \subseteq \mathcal{S}$, find a controller $\pi^{*} \left(a|s\right)$ such that all trajectories $s_{t}$ under $P\left(s_{t+1} \mid s_{t}, a_{t}\right)$, and $s_{0} \in \mathcal{S}_{\text{initial}} \subseteq \mathcal{S}$ have the following properties: \textbf{Reachability}: given a tolerance $\delta$, $\exists T \geq 0$, such that $\mathbb{E}_{s_{t}}\left\|s_{t}-s_{\text{goal}}\right\| \leq \delta$, $\forall t \geq t_{0}+T$; \textbf{Safety}: $\mathbb{P}(s_{t} \notin \mathcal{S}_{\text{unsafe}}\mid s_{0}, \pi, t)=\int_{\mathcal{S} \setminus \mathcal{S}_{\text{unsafe}}} p(s \mid s_{0}, \pi, t) \mathrm{d} s = 0$, $ \forall t \geq t_{0}$.
\end{definition}

The state $s_{\text{irrecoverable}} \in \mathcal{S}_{\text{\text{irrecoverable}}} \not\subset \mathcal{S}_{\text{unsafe}}$ are not themselves unsafe, but inevitably lead to unsafe states under the controller $\pi$. Thus, we also consider $s_{\text{irrecoverable}}$ to be unsafe for the given controller $\pi$.

\begin{definition}
\label{def:state}
A state is said to be \textbf{irrecoverable} if $s\notin \mathcal{S_{\text{unsafe}}}$ under the controller $a\sim\pi(a|s)$, the trajectory defined by $s_{0}=s$ and $s_{t+1}\sim P(s_{t+1}|s_{t},\pi(s_{t}))$ satisfies $\mathbb{P}(s_{t} \in \mathcal{S}_{\text{unsafe}}\mid s_{0}, \pi, t)=\int_{\mathcal{S}_{\text{unsafe}}} p(s \mid s_{0}, \pi, t) \mathrm{d} s \neq 0$, $\exists   \hat{t}>t_{0}$. 
\end{definition}

Therefore, the safety and unsafety of a certain state can be described as: the state $s\in \mathcal{S}_{\overline{\text{unsafe}}}=\mathcal{S}_{\text{irrecoverable}} \cup \mathcal{S}_{\text{unsafe}}$ is unsafe, while the state $s\in \mathcal{S}_{\overline{\text{safe}}} = \mathcal{S} \backslash \mathcal{S}_{\overline{\text{unsafe}}}$ is safe.

In reach-avoid problems, CLFs and CBFs are widely used to ensure reachability and safety of the system \cite{jin2020neural}, respectively. To avoid the conflicts between separate certificates, we rely on the CLBF, a single unifying certificate for both reachability and safety \cite{romdlony2016}. In this paper, the definition of the CLBF is related to \cite{dawson2022safe}. We extend it from a continuous-time system to CMDP (similar to the definition of CBF in discrete-time system \cite{agrawal2017discrete}).
 In CMDP, the definition of CLBF is given as follows.
\begin{definition}
\label{def:CLBF}
(CLBF). A function V: $\mathcal{S} \rightarrow \mathbb{R}$ is a CLBF, for some constant $\hat{c}$, $\lambda>0$, \textcircled{1} $V\left(s_{\text{goal}}\right)=0$, \textcircled{2} $V(s)>0, \forall s \in \mathcal{S} \backslash \mathcal{S}_{\text{goal}}$, \textcircled{3} $V(s)\geq \hat{c}, \forall s \in \mathcal{S}_{\overline{\text{unsafe}}}$, \textcircled{4} $V(s)<\hat{c}, \forall s \in \mathcal{S}_{\overline{\text{safe}}}$. \textcircled{5} there exists a controller $\pi$, such that $\mathbb{E}_{s^{\prime}} [V\left(s^{\prime}\right)-V(s)+\lambda V(s) ]\leq 0, \forall s \in \mathcal{S} \backslash s_{\text{goal}}$, where $s^{\prime} \sim P(s^{\prime}|s,\pi(s))$.
\end{definition}
Thus, any controller 
$\pi \in \{\pi \mid \mathbb{E}_{s^{\prime}} [V (s^{\prime})-V(s)+\lambda V(s) ]\leq 0, s^{\prime} \sim P(s^{\prime}|s,\pi(s))\}$
can satisfy reachability and safety \cite{dawson2022safe}. In this definition, the transition $P(s^{\prime}|s,\pi(s))$ requires the knowledge of a dynamic system model, but modeling error can hardly be avoided in reality. Next, we will show how we can use model-free RL to learn CLBFs and controllers with reachability and safety guarantee.

\section{Reinforcement Learning Algorithm with Safety Guarantee}
\label{sec:theory}
In an actor-critc framework, the high-level plan is as follows. We first choose the value function $V(s)$ to be the CLBF, similar to those done in approximate/adaptive dynamic programming~\cite{jiang2020learning} on choosing the Lyapunov function. Then we expect to impose some properties of CLBF as constraints in the Bellman recursion to find the value function (i.e., CLBF) and hope to search the corresponding policy, similar to what is done in \cite{berkenkamp2017safe, chow2018lyapunov,han2020actor}. Conceptually, we are interested in the following conceptual problem formulation:

\textbf{Repeat}
\begin{itemize}
    \item Find:  V.  \text{           Subject to: CLBF constraints}
    \item Find: $\pi$ using V
\end{itemize}

\textbf{Untill}  V, $\pi$ convergence. 

\subsection{CLBF as Critic}
\label{sec:clbf}
To enable the actor-critic learning, the control Lyapunov barrier critic $Q_{\text{LB}}$ is designed to be dependent on $s$ and $a$, while $V(s)=Q_{\text{LB}}(s,\pi_{\theta}(s))$.
Then we present a method to construct a $Q_{\text{LB}}$ through the Bellman recursion. The target function $Q_\text{target}$ is a valid control Lyapunov barrier critic which is approximated by:
\begin{equation}
    Q_{\text{target }}(s_{t}, a_{t})=c(s_{t}, a_{t})+\gamma Q_{\text{LB}}^{\prime}\left(s_{t+1}, \pi(s_{t+1})\right)
\end{equation}
where $Q_{\text{LB}}^{\prime}$ is the network that has the same structure as $Q_{\text{LB}}$, but parameterized by a different set $\phi^{\prime}$, as typically used in the actor-critic methods \cite{lillicrap2015continuous,haarnoja2018soft}. The parameter $\phi^{\prime}$ is updated through exponential moving average of weights controlled by a hyperparameter $\tau \in \mathbb{R}_{(0,1)}, \phi_{k+1}^{\prime} \leftarrow \tau \phi_{k}+(1-\tau) \phi_{k}^{\prime}$.

Such that the value function meets the requirements of our main theorem (Theorem \ref{theorem:clbf} in Section \ref{sec:result}), the tuples $\{s_{t},a_{t},c(s_{t},a_{t}),s_{t+1}\}$ are set as follows:
\begin{equation}
\label{eq:cost}
\begin{cases}
\{s_{t},a_{t},0, s_{t}\} & s_{t} \in \mathcal{S}_{\text{goal}} \\
\{s_{t},a_{t},c(s_{t}, a_{t}), s_{t+1}\} & s_{t} \in \mathcal{S}_{\text{safe}}\setminus \mathcal{S}_{\text{goal}} \\
\{s_{t},a_{t},C, s_{t}\} & s_{t} \in \mathcal{S}_{\text{unsafe }}
\end{cases}
\end{equation}
where the terminal cost $C$ is a constant. 

\subsection{Data-based CLBF Theorem}
\label{sec:result}
In this part, inspired by Definition~\ref{def:CLBF} of CLBF, we propose a novel data-based theorem, on which the constraints should be in the conceptual problem formulation at the beginning of Section~\ref{sec:theory}. Instead of explicitly using a dynamic model, the following theorem provides a sufficient condition for reachability and safety based on samples.

Before presenting the main theorem, we need the following Lemma~\ref{lemma1}, in addition to \eqref{eq:cost}, on the terminal cost $C$ to hold, so that $V(s_{\overline{\text{unsafe}}}) \geq \hat{c}$ and $V(s_{\overline{\text{safe}}}) < \hat{c}$, as required in (\ref{theorem: thm1_1}).

\begin{lemma}
\label{lemma1}
Suppose that $N$ is the maximum number of steps in each episode, let $C>\frac{c_{\max}(s,a)\left(1-\gamma^{N}\right)}{\gamma^{N}}$, when $\gamma<1$.
Under the controller $\pi$, if $s \in \mathcal{S}_{\overline{\text{unsafe}}}$\thinspace, $V(s) \geq \hat{c}$, and $s \in \mathcal{S}_{\overline{\text{safe}}}$ \thinspace, $V(s) < \hat{c}$.
\end{lemma}
\begin{proof}
The proof can be found in Appendix~\ref{ap:proof2}.
\end{proof}

\begin{theorem}
\label{theorem:clbf}
If there exists a function $V(s):\mathcal{S}\rightarrow \mathbb{R_+}$ and positive constants $\alpha_{1}$, $\alpha_{2}$, $\alpha _{3}$, $\alpha _{4}$, such that
\begin{align}
    \begin{aligned}
    \label{theorem: thm1_1}
        \alpha _{1}c_\pi(s) \leq &V(s) < \min{(\alpha _{2}c_\pi(s),\hat{c})} < \hat{c},   \ \ \ \ \ \ \  \forall s \in \mathcal{S}_{\overline{\text{safe}}}\\ 
        \hat{c} \leq &V(s)\leq \hat{c}+\alpha _{3}c_\pi(s) < (1+\alpha_{3})\hat{c},  \  \   \forall s \in \mathcal{S}_{\overline{\text{unsafe}}}
    \end{aligned}
\end{align}
and
\begin{align}
    \begin{aligned}
        &\mathbb{E}_{s\sim \mu_N}\left( \mathbb{E}_{s^{\prime }\sim P_{\pi
        }}V(s^{\prime })\mathds{1}_\Delta(s^{\prime })-V(s)\mathds{1}_\Delta(s)\right) \\
        &< -\alpha _{4}\mathbb{E}_{s\sim \mu_N}c_\pi(s)\mathds{1}_\Delta(s)
    \end{aligned}
    \label{theorem:clbf2}
\end{align}%
where $c_{\pi}\left(s_{t}\right) \triangleq \mathbb{E}_{a \sim \pi} c\left(s_{t}, a_{t}\right)$, and $c_{\pi}\left(s\right) \leq \hat{c}, \forall s \in\mathcal{S}$. The cost function $c(s_{t},a_{t})=\mathbb{E}_{P\left(\cdot \mid s_{t}, a_{t}\right)}\left\|s_{t+1}-s_{goal}\right\|$ describes the distance to the goal set. $\mu_N(s)$ denotes the average distribution of $s$ over the finite $N$ time steps,
\begin{equation*}
    \mu_N(s)\doteq\frac{1}{N}\sum_{t=1}^N p(s|\rho, \pi,t)
\end{equation*}
$N$ is the maximum number of steps in each episode. $\mathds{1}_\Delta(s)$ denotes the function;
\begin{equation*}
   \mathds{1}_\Delta(s) =\left\{
\begin{array}{cc}
1 & s\in\Delta \\
0 & s\notin\Delta%
\end{array}%
\right.
\end{equation*}
where $\Delta =\mathcal{S}\backslash \left( \mathcal{S}_{\text{goal}} \cup \mathcal{S}_{\text{unsafe}}\right)$, $\mathcal{S}_{\text{goal}}=\left\{s \mid c_{\pi}(s) \leq  \delta\right\}=\left\{s \mid \left\|s-s_{\text{goal}}\right\| \leq \delta \right\}$. Note that $c_{\pi}(s) > \delta$, $\forall s \in \Delta$.

Then the followings hold: i) if $s_{0} \in \mathcal{S}_{\overline{\text{safe}}}, V(s_{0}) \leq \hat{c}$, the system is reachable with tolerance $ \delta$ and safe within $N$ steps; ii) if $s_{0} \in \mathcal{S}_{\overline{\text{unsafe}}}, V(s_{0})>\hat{c}$, the agent would reach the unsafe areas within $N$ steps.
\end{theorem}
\begin{proof}
The proof can be found in Appendix~\ref{ap:proof1}.
\end{proof}

\subsection{Lyapunov Barrier Actor-Critic Algorithm}
\label{sec:lbac}
Recent advance in \cite{han2020actor} has guaranteed reachability by the Lagrangian relaxation method. Taking inspiration from their work, we extend to safety guarantee by designing an actor-critic RL algorithm. The proposed Algorithm~\ref{algo:LBAC} is named Lyapunov barrier actor-critic (LBAC), which gains a value function that satisfies the requirements of Theorem~\ref{theorem:clbf}, and a corresponding safe controller.

The control Lyapunov barrier critic function $Q_{\text{LB}}$ and the actor function (controller) $\pi_{\theta}(a_{t}|s_{t})$ are parametrized by $\phi$ and $\theta$, respectively.
Note that the stochastic controller $\pi_\theta$ is parameterized by a deep neural network $f_\theta$ that depends on $s$ and Gaussian noise $\epsilon$.
The goal is to construct the CLBF as the critic function with constraints~\eqref{theorem:clbf2} under the controller $\pi_{\theta}(a_{t}|s_{t})$.
By using the Lagrange relaxation technique \cite{bertsekas1997nonlinear}, $Q_{\text{LB}}$ is updated using gradient descent to minimize the following objective function
\begin{equation}
\begin{aligned}
\label{eq:critic objective}
J(\phi&)=\mathbb{E}_{\mathcal{D}}\left[\frac{1}{2}\left(Q_{\text{LB}}(s, a)-Q _{\text{target}}(s, a)\right)^{2}\right. \\
&\left.+\lambda(Q_{\text{LB}}(s',f_{\theta}(\epsilon,s'))\mathds{1}_\Delta(s')-Q_{\text{LB}}(s,a)\mathds{1}_\Delta(s)+ \alpha_{4}\hat{c})\right]
\end{aligned}
\end{equation}
where $Q_{\text{target}}$ is the approximation target related to the chosen control Lyapunov barrier candidate, $\lambda$ is a Lagrange multiplier that controls the relative importance of the inequality condition \eqref{theorem:clbf2}. $\mathcal{D}$ is the set of collected transition pairs that are determined in (\ref{eq:cost}) and Lemma \ref{lemma1}. The control Lyapunov barrier candidate acts as a supervision signal to the control Lyapunov barrier critic function.

LBAC is based on the maximum entropy framework \cite{haarnoja2018soft}, which can improve controller exploration during learning. A minimum entropy constraint is added to the above optimization problem to derive the following objective function
\begin{equation}
\begin{aligned}
\label{eq:LBAC}
 J(\theta) =& \mathbb{E}_{(s,a,s',c)\sim\mathcal{D}}[ Q_{LB}(s, f_{\theta}(\epsilon,s))\\
 &+ \beta (\log(\pi_\theta(f_\theta(\epsilon,s)|s))+\mathcal{H}_t))] 
\end{aligned}
\end{equation}
where $\beta$ is a Lagrange multiplier that controls the relative importance of the minimum entropy constraint, $\mathcal{H}_t$ is the desired entropy bound.

In the actor-critic framework, the parameters of the controller are updated through stochastic gradient descent, which is approximated by 
\begin{equation}
\begin{aligned}
\nabla_\theta J(\theta)
&= \beta \nabla_\theta  \log(\pi_\theta(a|s))+  \beta \nabla_a  \log(\pi_\theta(a|s))\nabla_\theta f_\theta(\epsilon,s)\\
&+ \nabla_{a'}Q_{\text{LB}}(s',a')\nabla_\theta f_\theta(\epsilon,s')
\end{aligned}
\label{LBAC_policy_gradient}
\end{equation}

Finally, the values of Lagrange multipliers $\lambda$ and $\beta$ are adjusted by gradient ascent to maximize the following objectives, respectively,
\begin{equation}
\begin{aligned}
J(\lambda)&=\lambda \mathbb{E}_{\mathcal{D}_{\Delta}}\left[Q_{\text{LB}}\left(s^{\prime}, f_{\theta}\left(s^{\prime}, \epsilon\right)\right) \mathds{1}_{\Delta}\left(s^{\prime}\right) \right.\\
&\left. -\left(Q_{\text{LB}}(s, a)-\alpha_{4} \hat{c}\right) \mathds{1}_{\Delta}(s)\right],\\
J(\beta)&=\beta \mathbb{E}_{\mathcal{D}}\left[\log \pi_{\theta}(a \mid s)+\mathcal{H}_{t}\right]
\label{eq:lagrange}
\end{aligned}
\end{equation}

During training, the Lagrange multipliers are updated by 
$$\lambda \leftarrow \max (0,\lambda +  \bar{\delta} \nabla_\lambda J(\lambda)), \ \ \ \ 
\beta \leftarrow \max (0,\beta +  \bar{\delta} \nabla_{\beta} J(\beta)) $$
where $\bar{\delta}$ is the learning rate. 
The pseudocode of the proposed algorithm is shown in Algorithm~\ref{algo:LBAC}.
\\
\\
\vspace{-1cm}
\begin{algorithm}[!]
   \caption{Lyapunov Barrier Actor-Critic (LBAC)}
   \label{algo:LBAC}
\begin{algorithmic}
    \REQUIRE 
    Maximum episode length $N$; maximum iteration steps $M$ 
   \REPEAT
   \STATE Sample $s_0$ according to $\rho$
   \FOR{$t=0$ to $N$}
   \STATE Sample $a_{t}$ from $\pi_\theta(a_{t}|s_{t})$ and step forward
   \STATE Observe $s_{t+1}$, $c_{t}$ and store $(s_{t},a_{t},c_{t},s_{t+1},\mathcal{I})$ in $\mathcal{D}$
   \ENDFOR
   \FOR{$i=1$ to $M$}
    \STATE Sample mini-batches of transitions from $\mathcal{D}$ and update $Q_{\text{LB}}$, $\pi$, Lagrange multipliers with \eqref{eq:critic objective}, \eqref{eq:LBAC}, \eqref{eq:lagrange}
   \ENDFOR
   \UNTIL{\eqref{theorem:clbf2} is satisfied}
\end{algorithmic}
\end{algorithm}

\section{Results and Validation}
\label{sec:experiment}
In this section, we consider a 2D quadrotor navigation task, i.e., aiming to reach a target while avoiding obstacles, as illustrated in Figure \ref{fig:experiment_setup}. The experiment setup is detailed in Appendix~\ref{env:2D}. First, we show separate CLFs and CBFs can lead to local optimums by implementing a CLF-CBF based Quadratic Program (CLF-CBF-QP). Then, we show the effectiveness of the proposed LBAC algorithm and evaluate it in the following aspects:
\begin{itemize}
    \item Training convergence: does the proposed training algorithm converge with random parameter initialization;
    \item Validation of CLBF: how do the learned CLBFs fit the goal and obstacles in the 2D quadrotor navigation task, and does the reachability and safety condition, i.e., Theorem \ref{theorem:clbf}, hold for the learned controllers;
    \item Sim-to-Real transfer: can we transfer the simulation training result directly to real-world robots, e.g., using a CrazyFlie 2.0 quadrotor.
\end{itemize}

In this part, the performance of LBAC on the CMDP tasks is evaluated compared with Risk Sensitive Policy Optimization (RSPO) \cite{geibel2005risk}, Safety Q-Functions for RL (SQRL) \cite{srinivasan2020learning}, and Reward Constrained Policy Optimization (RCPO) \cite{tessler2018reward}. We use the public codebase of \cite{thananjeyan2021recovery} to implement the comparison experiments.
The hyperparameters are described in Appendix~\ref{ap:hyperpara}.

\subsection{Conflicts between CLFs and CBFs}
\label{sec:clfcbf}
To show there exist conflicts between CLFs and CBFs as separate certificates, we implemented a model-based CLF-CBF-QP controller \cite{ames2014control} which incorporates a CLF and CBFs as constraints through quadratic programs. As shown in Figure~\ref{fig:drone_stuck_b}, the quadrotor easily gets stuck before the wall which is in front of the target. This is because the attraction of the CLF is balanced by the repulsion of CBFs, as illustrated by Figure \ref{fig:drone_stuck_a}. The quadrotor can still successfully reach the target if it luckily avoids conflicting areas. We also tried CLFs and CBFs as separate critics in a multi-objective RL setting, but failed to converge. The failure of the above CLF-CBF controllers motivates our CLBF approach which satisfies both safety and reachability in this 2D quadrotor navigation task, as illustrated in Figure~\ref{fig:CLBF gradient}.

\begin{figure}
    \vspace{0.2cm}
    \centering
    \subfigure[Illustration of CLF-CBF]
    {
    \includegraphics[width=0.45\columnwidth]{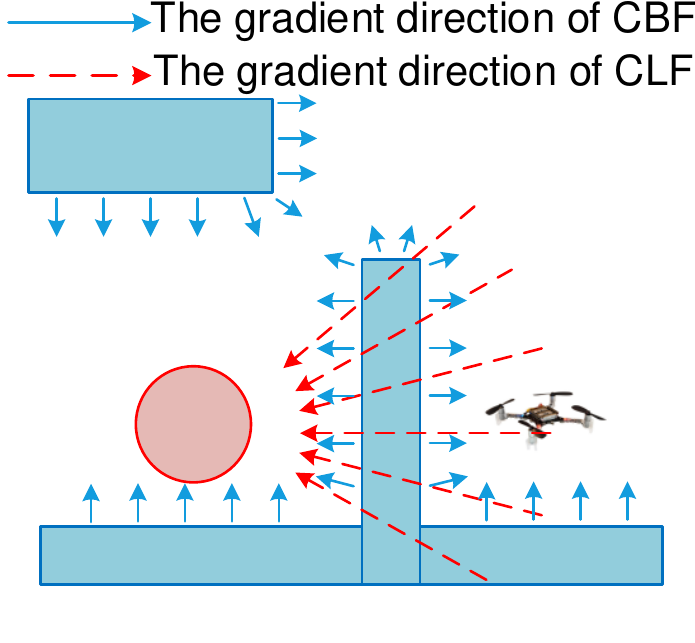}
    \label{fig:drone_stuck_a}
    }
    \hspace{-0.2cm}
    \centering
    \subfigure[Results of CLF-CBF-QP]
    {
    \includegraphics[width=0.47\columnwidth]{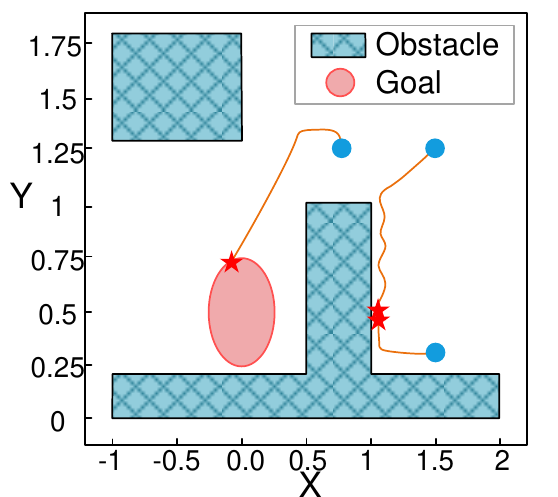}
    \label{fig:drone_stuck_b}
    }
    \vspace{-0.2cm}
    \caption{Performance of a CLF-CBF-QP controller. (a) is an intuitive illustration of CLF-CBF. In (b), lines are trajectories. The blue circles stand for the starting points. The red stars represent the final position.}
    \label{fig:clf-cbf}
    \vspace{-0.2cm}
\end{figure}

\subsection{Training Convergence}
The main criterion we are interested in is the convergence of the controller during the training process. Each approach is trained with five different random seeds. The total cost and number of violations during training are plotted in Figure~\ref{fig:training process}. Among the RL algorithms to be compared, LBAC, RSPO, and SQRL can converge within $2300$ episodes, while RCPO fails to converge even in $3000$ episodes. As shown in Figure~\ref{fig:training process}, LBAC leads to a fewer number of violations during training than other model-free safe RL methods. 
\begin{figure}[htb]
    \vspace{-0.4cm}
    \centering
    \subfigure[Total Cost]{
    \includegraphics[width=0.49\columnwidth]{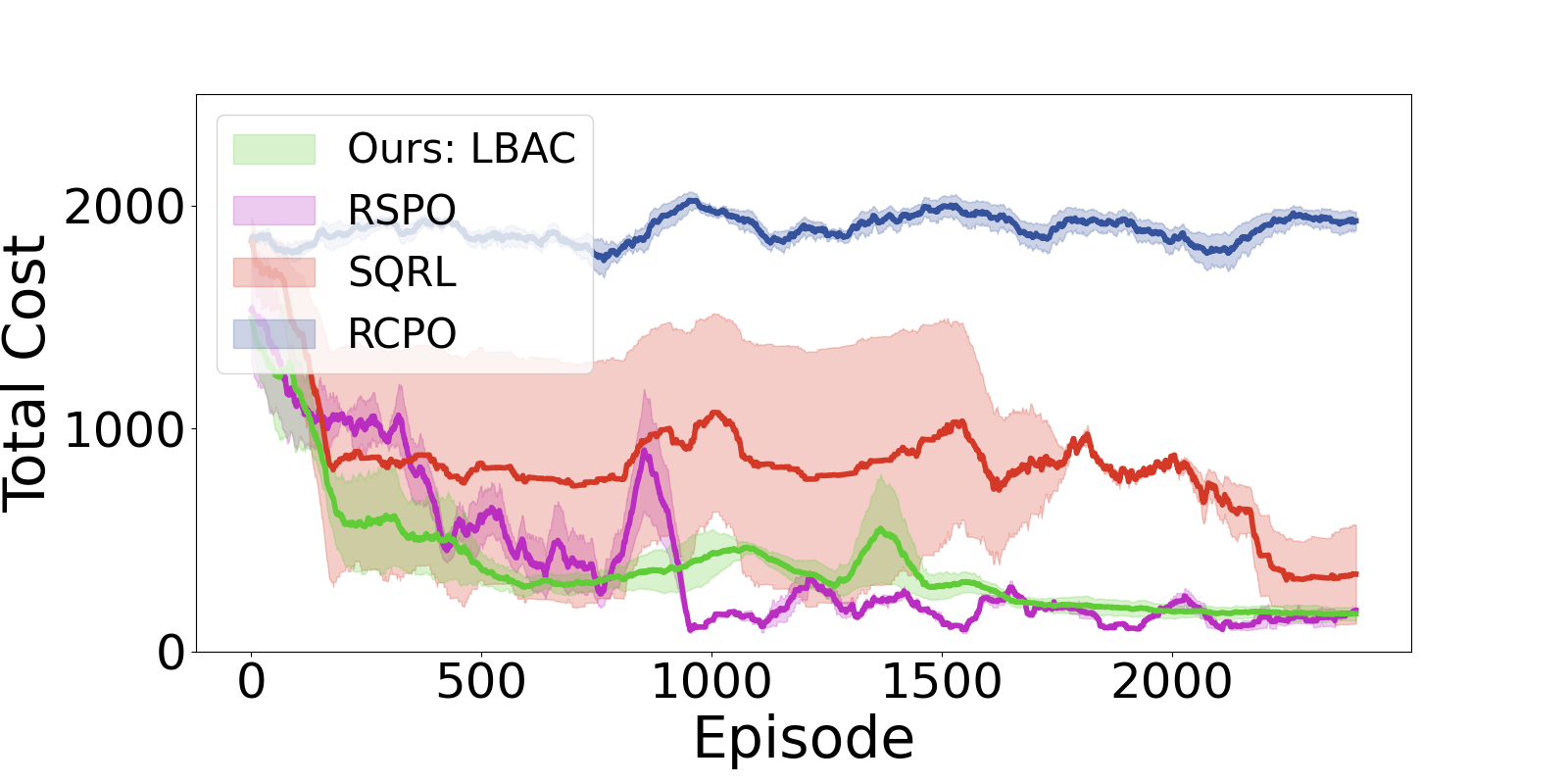}
    \label{fig:kin_car_violation}
    }
    \hspace{-0.7cm}
    \centering
    \subfigure[Training-time violations]{
    \includegraphics[width=0.49\columnwidth]{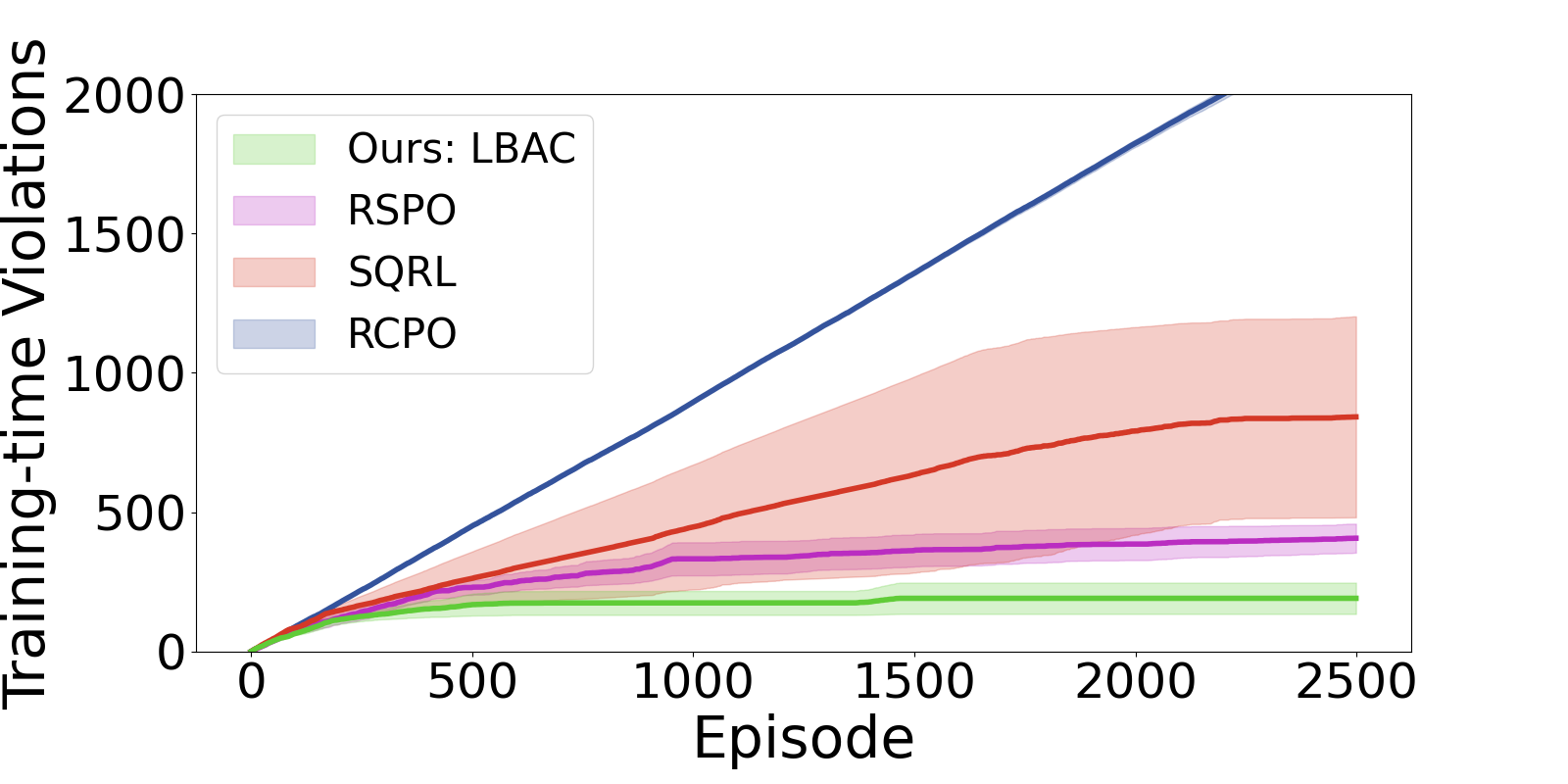}
    \label{fig:2-D navigation_violation}
    }
    \vspace{-0.2cm}
    \caption{{Total cost and the number of violations during training. The Y-axis indicates the total cost in one episode in (a) and total violation times during training in (b). The X-axis indicates the total episodes. The shaded region shows the 1-SD confidence interval of five random seeds.}}
    \vspace{-0.2cm}
    \label{fig:training process}
\end{figure}

\subsection{Validation of CLBF}
In this part, we examine the learned control Lyapunov barrier critic function. We pick the controllers and corresponding CLBFs trained in 1000, 1500, and 2000 episodes. The contour plots of the CLBFs are shown in Figure~\ref{fig:CLBF_main} as a function of $x$ and $y$, where $\{v_{x},v_{y}\}$ is set to $\{0,0\}$. The white lines are the safety boundaries of the CLBFs, i.e. when $V(s)=\hat{c}$ and $\hat{c}$ is set 2000.
As shown in Figure~\ref{fig:CLBF_main}, we find that the safety boundary of CLBF where $V(s)=\hat{c}$ gradually approaches the obstacle boundary with increasing training episodes. However, we also noticed some unsafe corner cases are considered as safe (such as the bottom right corner of the left obstacle). This could be due to the exploration and exploitation dilemma LBAC suffers as a model-free RL algorithm.

\begin{figure}[htb]
    \vspace{0.2cm}
    \centering
    \includegraphics[width=0.98\columnwidth]{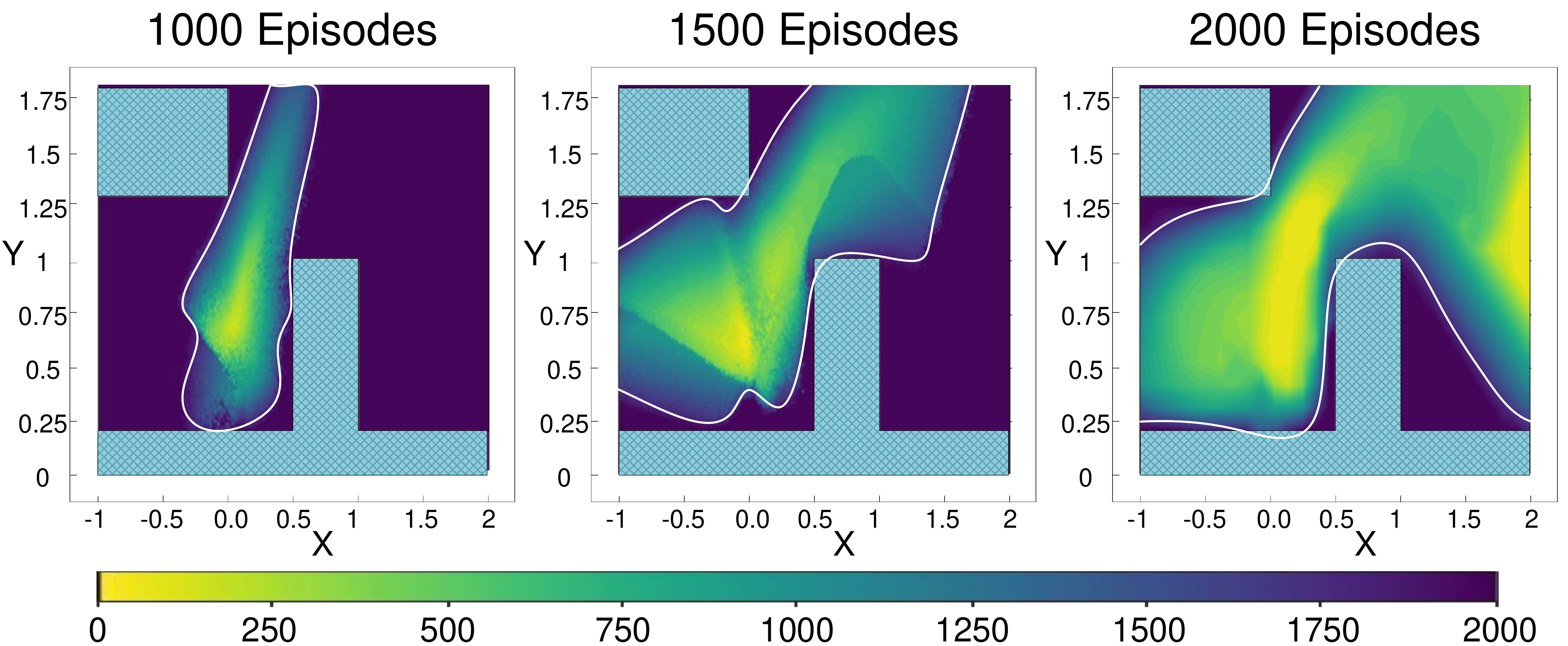}
    \vspace{-0.2cm}
    \caption{The contour plots of the CLBF. The white lines show the contour of the learned CLBF. The color bar denotes the function value. From left to right, the contour plots are the CLBFs trained in 1000 episodes, 1500 episodes and 2000 episodes.}
\label{fig:CLBF_main}
\vspace{-0.2cm}
\end{figure}

We also validate the learned CLBF by showing the outcomes of the trajectory rollouts starting from uniformly sampled initial positions. This is because of the well-known fact that it is challenging to initialize uniformly throughout the state space in a model-free setting. For example, we can hardly make a robot have a specific velocity at a particular position. Figure~\ref{fig:drone_traj_rollout} shows that the quadrotors starting from the unsafe region would be violating, while those that start in the safe region would successfully reach the goal. We present the changes in CLBF values along the trajectories in Figure~\ref{fig:drone_traj_clbf_a}, and the averaged changes in CLBF value of these trajectories in Figure~\ref{fig:drone_traj_clbf_b}. We can observe that the averaged value has a decreasing trend, which aligns with the theory before. These results indicate that the learned CLBF is valid.

\begin{figure}[htb]
    \vspace{-0.2cm}
    \centering
    {
    \includegraphics[width=0.98\columnwidth]{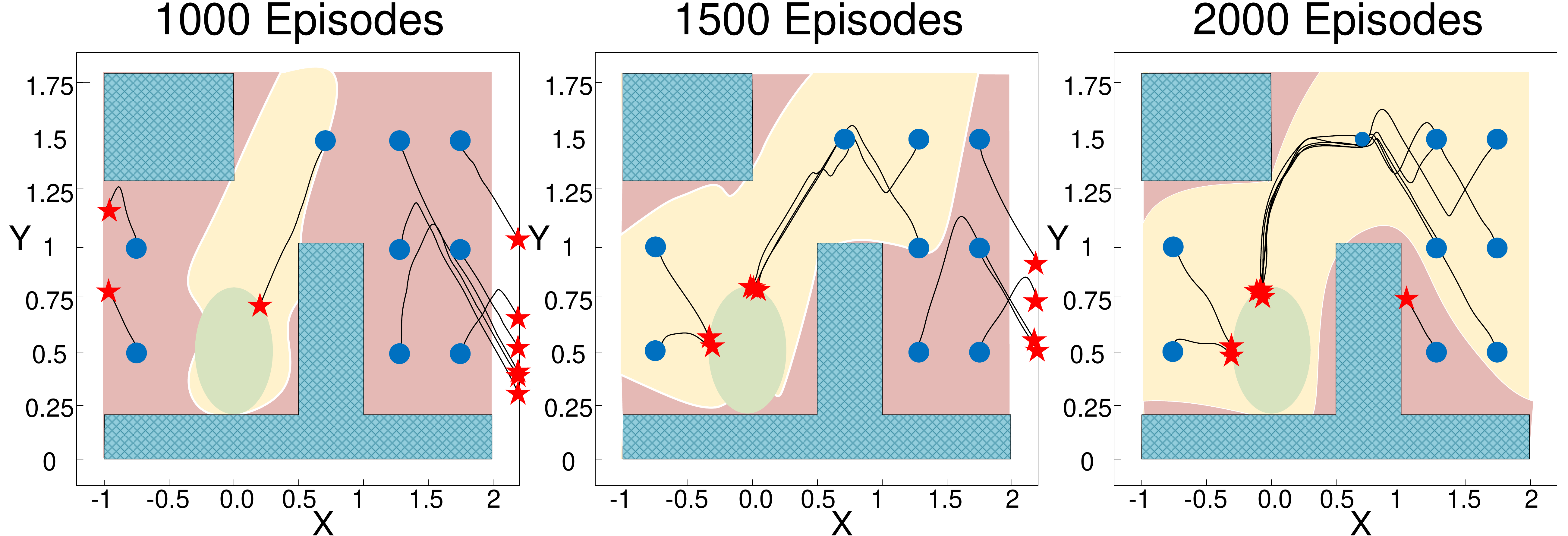}
    }
    \vspace{-0.5cm}
    \caption{Trajectories of the learned LBAC controllers in the simulator. The shaded area corresponds to the unsafe region. The green ellipse area stands for the goal. The blue circles are the initial positions, while the red stars are the end positions.}
    \label{fig:drone_traj_rollout}
\end{figure}
\vspace{-0.4cm}

\begin{figure}[htb]
    \vspace{-0.2cm}
    \centering
    \subfigure[The changes in CLBF value]
    {
    \includegraphics[width=0.45\columnwidth]{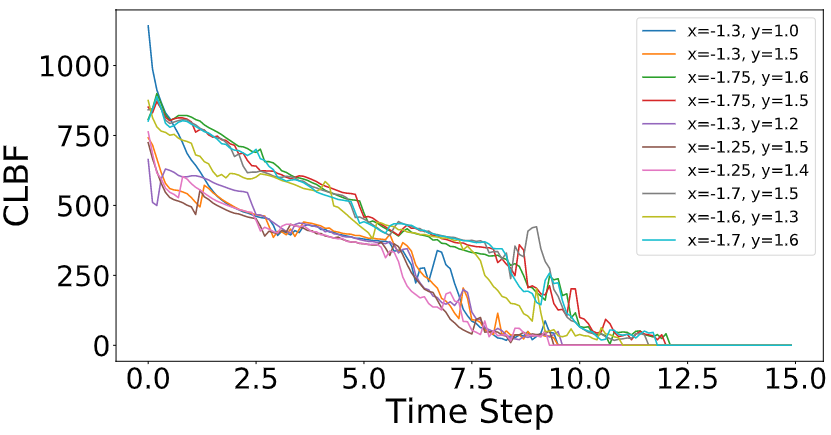}
    \label{fig:drone_traj_clbf_a}
    }
    \centering
    \subfigure[The averaged changes]
    {
    \includegraphics[width=0.45\columnwidth]{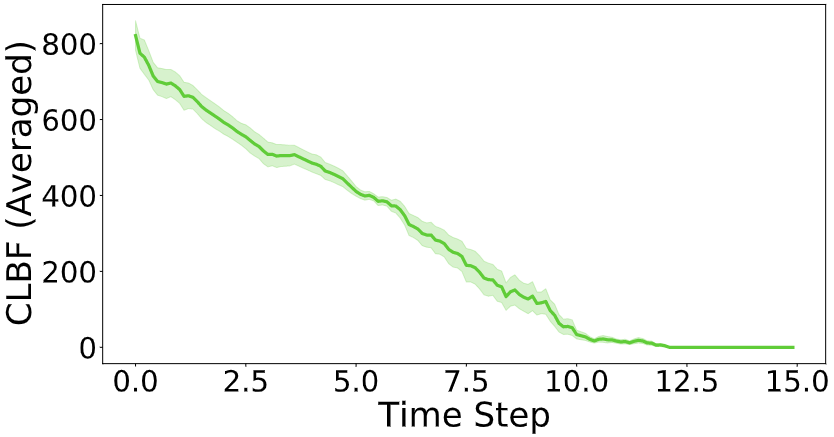}
    \label{fig:drone_traj_clbf_b}
    }
    \vspace{-0.2cm}
    \caption{The changes in CLBF value under different initial conditions. In (a), we show the changes in CLBF value along the trajectories starting from ten different initial positions. In (b), the averaged change in CLBF value of these trails is plotted. The solid line indicates the average value and shadowed region for the 1-SD confidence interval of these trails.}
    \label{fig:drone_traj_clbf}
\end{figure}
\vspace{-0.4cm}

\subsection{Sim-to-Real Transfer}
\label{sec:sim-to-real}
In this part, we evaluate LBAC by deploying controllers learned in the simulators to the physical robot. As shown in Figure~\ref{fig:experiment_setup}, a nano Crazyflie 2.0 quadrotor is used to achieve the autonomous navigation task and a motion capture system is used for state estimation in the real world. The trajectories of the Crazyflie starting from different initial positions are shown in Figures~\ref{fig:trajectory0.35} and \ref{fig:trajectory1.3}. The controllers trained by LBAC outperform other model-free safe RL algorithms in terms of both reachability and safety.

\begin{figure}[htb]
    \centering
    \subfigure[Illustration of CLBF]{
    \includegraphics[width=0.31\columnwidth]{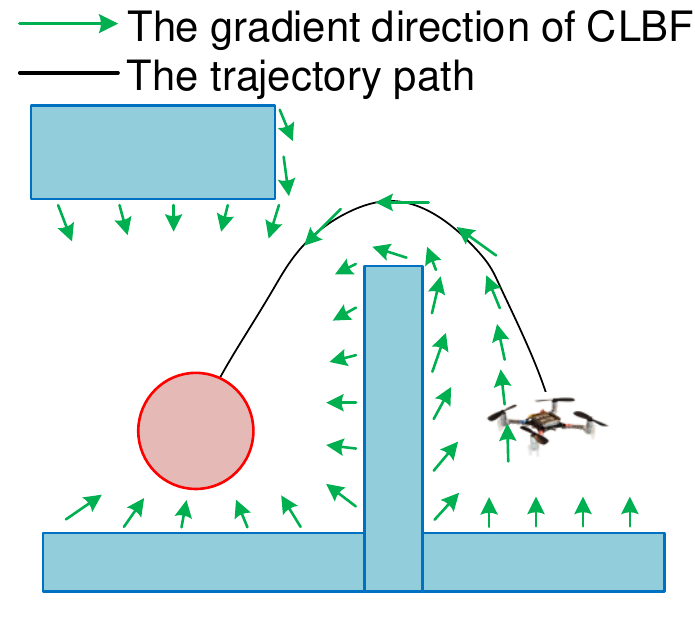}
    \label{fig:CLBF gradient}
    }
    \hspace{-0.34cm}
    \centering
    \subfigure[Height 0.35m]
    {
    \includegraphics[width=0.31\columnwidth]{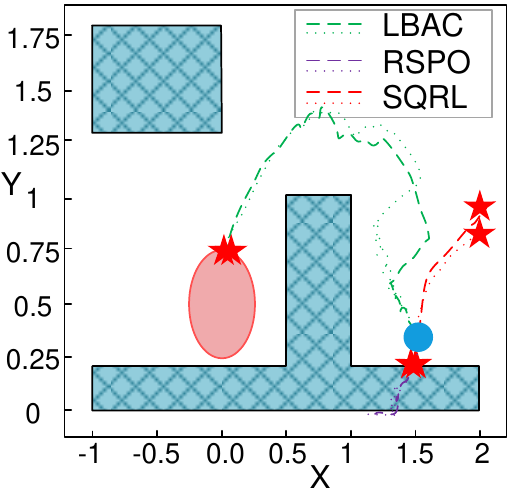}
    \label{fig:trajectory0.35}
    }
    \hspace{-0.34cm}
    \centering
    \subfigure[Height 1.2m]{
    \includegraphics[width=0.31\columnwidth]{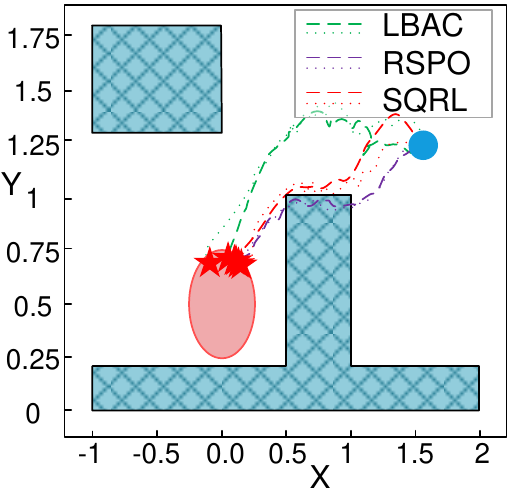}
    \label{fig:trajectory1.3}
    }
    \vspace{-0.2cm}
    \caption{Controllers are evaluated in real-world using a Crazyflie 2.0 quadrotor. (a) is an intuitive illustration of CLBF. In (b) and (c), the quadrotor's initial heights are 0.35m and 1.2m. The blue circle represents the starting points, and the red stars stand for the reached positions.}
    \label{fig:experiment}
\end{figure}
\vspace{-0.2cm}

\section{Conclusion}
\label{sec:conclusion}
In this paper, the control Lyapunov barrier function is extended to the constrained Markov decision process, and a data-based theorem is proposed to analyze closed-loop reachability and safety. Based on the theoretical results, a Lyapunov Barrier-based Actor-Critic method is proposed to search for a controller. The proposed algorithm is evaluated on a 2D quadrotor navigation task with safety constraints. Compared to existing model-free RL algorithms, the proposed method can reliably ensure reachability and safety in both simulation and real-world tests. In the future, more experiments will be conducted to validate the effectiveness and scalability of our approach. We also plan to improve the robustness of the learned controller using methods such as domain randomization and adversarial training \cite{brunke2021safe}.

\begin{appendices}

\section{Proof of Lemma~\ref{lemma1}} 
\label{ap:proof2}
\begin{proof}
When $\hat{s} \in \mathcal{S}_{\overline{\text{safe}}}$, it leads to the goal state within $N$ steps. Thus, $V(\hat{s}) =\underset{a \sim \pi}{\mathrm{E}}\left[\sum_{t=0}^{\infty} \gamma^{t} c(s_{t},a_{t}) \mid s_{0}=\hat{s}\right] < \sum_{t=0}^{N-1} \gamma^{t} c_{\max}(s,a)=\frac{c_{\max}(s,a)\left(1-\gamma^{N}\right)}{1-\gamma} $. In order to have $V(\hat{s}) < \hat{c}$, we set $\frac{c_{\max}(s,a)\left(1-\gamma^{N}\right)}{1-\gamma} < \hat{c}$.
When $\hat{s} \in \mathcal{S}_{\overline{\text{unsafe}}}$, it leads to unsafe state within $N$ steps.
Thus, $V(\hat{s}) \geq \sum_{t=0}^{N-1} \gamma^{t} c_{\min}(s,a)+\sum_{t=N}^{\infty} \gamma^{t}C=\frac{c_{\min}(s,a)\left(1-\gamma^{N}\right)+C \gamma^{N}}{1-\gamma}$.
In order to have $V(\hat{s}) \geq \hat{c}$, we set $\frac{c_{\min}(s,a)\left(1-\gamma^{N}\right)+C \gamma^{N}}{1-\gamma} \geq \hat{c}$.
Rearranging, we have $C\geq\frac{(1-\gamma)\hat{c}-c_{\min}(s,a)\left(1-\gamma^{N}\right)}{\gamma^{N}}$. With $c_{\min}(s,a)=0$, it is simplified to $C\geq\frac{1-\gamma}{\gamma^{N}} \hat{c} > \frac{c_{\max}(s,a)\left(1-\gamma^{N}\right)}{\gamma^{N}}$. 
To this end, the condition \eqref{theorem: thm1_1} is achieved. 
\end{proof}

\section{Proof of Theorem~\ref{theorem:clbf}} \label{ap:proof1}

\begin{proof}
To prove that $N$ is finite based on the conditions and assumptions where $N=\max \{t: \mathbb{P}(s \in \Delta|\rho, \pi, t)>0\}$, we will assume that $N$ is infinity and prove by contradiction. $N=\infty$ if for any $\epsilon$ there exists an instant $t>\epsilon$ such that $\mathbb{P}(s\in \Delta|\rho, \pi, t)>0$.
In that case, the finite-horizon sampling distribution $\mu_N(s)$ turns into the infinite-horizon sampling distribution $\mu(s)= \lim_{N\rightarrow \infty}\mu_N(s)=\lim_{N\rightarrow \infty}\frac{1}{N}\sum_{t=1}^N p(s|\rho, \pi,t)$.
The existence of $\mu(s)$ is guaranteed by the existence of $q_\pi(s)=\lim _{t \rightarrow \infty} p(s \mid \rho, \pi, t)$, which has been commonly exploited by many RL literature \cite{han2020actor,han2021reinforcement}. Since the sequence $\{p(s|\rho,\pi,t), t\in\mathbb{Z}_+\}$ converges to $q_\pi(s)$ as $t$ approaches $\infty$, then by the Abelian theorem, the sequence $\{\frac{1}{T}\sum_{t=1}^T p(s|\rho,\pi,t), T\in\mathbb{Z}_+\}$ also converges and $\mu(s) = q_\pi(s)$. 
Then one naturally has that the sequence $\{\mu_N(s)V(s), T\in\mathbb{Z}_+\}$ converges pointwise to $q_\pi(s)V(s)$. 

According to Lebesgue's dominated convergence theorem \cite{royden1988real}, if a sequence ${f_n(s)}$ converges point-wise to a function $f$ and is dominated by some integrable function $g$ in the sense that,$\vert f_n(s) \vert \leq g(s), \forall s\in \mathcal{S},\forall n$, then one has $\lim_{n\rightarrow\infty}\int_\mathcal{S}f_n(s)\mathrm{d}s= \int_\mathcal{S}\lim_{n\rightarrow\infty}f_n(s)\mathrm{d}s$.

Applying this theorem to the left-hand side of \eqref{theorem:clbf2}

\begin{equation}
\begin{aligned}
    &\mathbb{E}_{s\sim \mu}\left(\mathbb{E}_{s^{\prime}\sim p_{\pi}}V(s^\prime)\mathds{1}_\Delta(s^\prime)-V(s)\mathds{1}_\Delta(s)\right)\\
    =&\int_\mathcal{S}\lim_{N\rightarrow\infty}\frac{1}{N}\sum_{t=1}^N p(s|\rho, \pi,t)(\int_\mathcal{S} p_{\pi}(s^{\prime}|s)V(s^{\prime})\mathds{1}_\Delta(s^\prime)\mathrm{d}s^{\prime}\\
    &-V(s)\mathds{1}_\Delta(s))\mathrm{d}s\\
    = &\lim_{N\rightarrow\infty}\frac{1}{N}\sum_{t=1}^N \int_\mathcal{S}V(s^{\prime})\mathds{1}_\Delta(s^\prime) \int_\mathcal{S}p_{\pi}(s^{\prime}|s)p(s|\rho, \pi,t) \mathrm{d}s\mathrm{d}s^{\prime}\\
    - &\lim_{N\rightarrow\infty}\frac{1}{N}\sum_{t=1}^N \int_\mathcal{S}p(s|\rho, \pi,t)V(s)\mathds{1}_\Delta(s) \mathrm{d}s\\
     =&\lim_{N\rightarrow\infty}\frac{1}{N}(\sum_{t=2}^{N+1} \mathbb{E}_{ p(s|\rho, \pi,t)}V(s)\mathds{1}_\Delta(s)\\
     &-\sum_{t=1}^{N} \mathbb{E}_{ p(s|\rho, \pi,t)}V(s)\mathds{1}_\Delta(s))\\
     =&\lim_{N\rightarrow\infty}\frac{1}{N}(\mathbb{E}_{p(s|\rho, \pi,N+1)}V(s)\mathds{1}_\Delta(s)- \mathbb{E}_{ \rho(s)}V(s)\mathds{1}_\Delta(s))
\label{*}
\end{aligned}
\end{equation}
Since $\mathbb{E}_{ \rho(s)}V(s)$ is finite, thus the limitation value $\lim_{N\rightarrow\infty}\frac{1}{N}(\mathbb{E}_{ \rho(s)}V(s)\mathds{1}_\Delta(s))=0$. The above equation equals to $\lim_{N\rightarrow\infty}\frac{1}{N}\mathbb{E}_{ p(s|\rho, \pi,N+1)}V(s)\mathds{1}_\Delta(s)$.
Note that $V(s) \geq \alpha_{1}c_{\pi}(s)$, $\forall s \in\mathcal{S}$, and $c_{\pi}(s) > \delta$, $\forall s \in \Delta$. Thus, $\lim_{N\rightarrow\infty}\frac{1}{N}\mathbb{E}_{ p(s|\rho, \pi,N+1)}V(s)\mathds{1}_\Delta(s)
    \geq\lim_{N\rightarrow\infty}\frac{\alpha_1\delta}{N}\mathbb{E}_{ p(s|\rho, \pi,N+1)}\mathds{1}_\Delta(s) = 0$

Since $\mu(s)=q_\pi(s)$, 
the right-hand side of \eqref{theorem:clbf2} 
equals to 
$-\alpha_{4}\mathbb{E}_{s\sim q_\pi}c_\pi(s)\mathds{1}_\Delta(s) \leq-\alpha_{4}\mathbb{E}_{s\sim q_\pi}\delta\mathds{1}_\Delta(s)=-\alpha_{4}\delta\lim_{t\rightarrow\infty}\mathbb{P}(s\in\Delta|\rho, \pi,t)$.
Combining the above inequalities with \eqref{theorem:clbf2}, one has $\lim_{t\rightarrow\infty}\mathbb{P}(s\in\Delta|\rho, \pi,t)<0$, which is contradictory to the fact that $\mathbb{P}(s\in\Delta|\rho, \pi,t)$ is nonnegative. Thus there exist a finite $N$ such that $\mathbb{P}(s\in\Delta|\rho, \pi,t) =0$ for all $t>N$. In other word, the agent will reach the goal region or the unsafe region within $N$ steps.
According to \eqref{theorem: thm1_1}, $s_{0} \in \mathcal{S}_{\overline{\text{safe}}}, V(s_{0}) < \hat{c}$ where the agent will reach the goal region and avoid the unsafe region, while $s_{0} \in \mathcal{S}_{\overline{\text{unsafe}}}, V(s_{0}) \geq \hat{c}$ where the agent will reach the unsafe region within $N$ steps. The process of building such function $V$ is described in Section~\ref{sec:clbf}.
\end{proof}

\section{2D Quadrotor Navigation}
\label{env:2D}
The state of the 2D quadrotor model is defined as $s = [p_x, p_y, v_x, v_y]$, with control input $a = [v_{x_{\text{des}}},v_{y_{\text{des}}}]$. In this experiment, the controller is expected to navigate a 2D quadrotor to the goal set $\boldsymbol{S}_{\text{goal}}$ without colliding with the obstacles. We define the state space as $\mathcal{S}=\left\{\boldsymbol{s}: \boldsymbol{s}_{\mathrm{lb}} \leq \boldsymbol{s} \leq \boldsymbol{s}_{\mathrm{ub}}\right\}$ with $\boldsymbol{s}_{\mathrm{lb}}=[-1,0,-0.25,-0.25]$ and $\boldsymbol{s}_{\mathrm{ub}}=[2,1.8,0.25,0.25]$, representing the lower bound and upper bound of the set of the valid states. The action space is set as $\mathcal{A}=\left\{\boldsymbol{a}: -\boldsymbol{a}_{\mathrm{b}} \leq \boldsymbol{a} \leq \boldsymbol{a}_{\mathrm{b}}\right\}$ with $\boldsymbol{a}_{\mathrm{b}}=[0.25,0.25]$, by considering the real world hardware limitation. The cost function is designed as $c = \sqrt{4p_x^{2}+(p_y-0.5)^{2}}$. 
We set the obstacle set $\mathcal{S}_{o1}=\left\{\boldsymbol{s}: 0.5 \leq p_x \leq 1 , 0.2 \leq p_y \leq 1 \right\}$, $\mathcal{S}_{o2}=\left\{\boldsymbol{s}: -1 \leq p_x \leq 0 , 1.3 \leq p_y \leq 1.8 \right\}$ and $\mathcal{S}_{o3}=\left\{\boldsymbol{s}: p_z \leq 0.2 \right\}$, the unsafe state set $\mathcal{S}_{\text{unsafe}}=\left\{\boldsymbol{s}: \mathcal{S}_{o1} \cup \mathcal{S}_{o2} \cup \mathcal{S}_{o3}\right\}$, the goal state set $\mathcal{S}_{\text{goal}}=\left\{\boldsymbol{s}: \sqrt{p_x^{2}+(p_y-0.5)^{2}} \leq 0.3 \right\}$. Once the quadrotor reaches the $\mathcal{S}_{\text{unsafe}}$, the episode ends in advance and the cost function is set as $C=2000$. The episodes are of maximum length 200 and time step $dt = 0.1$ s. In the experiments, we use Bitcraze’s Crazyflie 2.0 quadrotors. We train the controllers in the simulator gym-pybullet-drones \cite{panerati2021learning} based on PyBullet. In the real world, we use a motion capture system for state estimation.

\section{Hyperparameter setting} \label{ap:hyperpara}
For LBAC, there are two networks: the controller network (actor) and the control Lyapunov barrier network (critic). The controller network is represented by a fully-connected neural network with two hidden layers of size 256 each, with the ReLU activation function, outputting the mean and standard deviations of a Gaussian distribution. A fully-connected neural network represents the control Lyapunov barrier critic network with two hidden layers of size 256, each with a ReLU activation function. We use the vanilla Soft Actor-Critic algorithm \cite{haarnoja2018soft} for 500 episodes to explore the environment effectively as a warm start. The hyperparameters can be found in Table~\ref{tab:hyper}

{\tiny
\begin{table}[htb]
\vspace{-0.1cm}
\caption{Hyperparameter setting in LBAC}\label{tab:hyper}
\vspace{-0.3cm}
\begin{center}
\begin{tabular}{l|c c c c}
Hyperparameters     & 2D Quadrotor Navigation\\\hline
    Minibatch size           & 512 \\
    Total episode & 2500 \\
    Actor learning rate     &$3\times10^{-4}$  \\
    Critic learning rate    &  $3\times10^{-4}$\\
    Terminal cost $C$    & 2000\\
    Discount factor $\gamma$   &0.999\\
\end{tabular}
\end{center}
\vspace{-0.1cm}
\end{table}
}

In RSPO and SQRL, another safety critic network $Q_{\text{risk}}$ is needed to estimate the discounted future probability of constraint violation with discounted $\gamma_{\text {risk}}$. The safety threshold $\varepsilon_{\text {risk }} \in[0,1]$ is an upper-bound on the expected risk of the action. In this paper, the safety critic network shares the same architecture as the task critic network, except that a sigmoid activation is added to the output layer to ensure that the outputs are on $[0,1]$. We use the same hyperparameter settings as LBAC in RSPO, RCPO, and SQRL. The other hyperparameters can be found in Table~\ref{tab:hyper2}. 

{\tiny
\begin{table}[htb]
\vspace{-0.1cm}
\caption{Hyperparameter setting in safe RL}\label{tab:hyper2}
\vspace{-0.3cm}
\begin{center}
\begin{tabular}{l|c c c c}
Hyperparameters      & 2D Quadrotor Navigation\\\hline
    RCPO $\left(\gamma_{\text {risk }}, \lambda\right) $        & $\left(0.99, 3000\right) $ \\
    RSPO $\left(\gamma_{\text {risk }}, \varepsilon_{\text {risk }}, \lambda\right) $         & $\left(0.99, 0.2, 10000\right) $\\
    SQRL $\left(\gamma_{\text {risk }}, \varepsilon_{\text {risk }}, \lambda\right) $         & $\left(0.99, 0.2,5000\right) $ \\
\end{tabular}
\end{center}
\vspace{-0.1cm}
\end{table}
}

\end{appendices}

\newpage
\bibliographystyle{IEEEtran}
\bibliography{ref}

\end{document}